\documentclass{article}

\PassOptionsToPackage{square,sort,comma,numbers}{natbib}

\usepackage[final]{neurips_2023}

\usepackage[utf8]{inputenc} 
\usepackage[T1]{fontenc}    
\usepackage{hyperref}       
\usepackage{url}            
\usepackage{booktabs}       
\usepackage{amsfonts}       
\usepackage{nicefrac}       
\usepackage{microtype}      
\usepackage{xcolor}         

\usepackage{blindtext}
\usepackage{amsmath, latexsym}
\usepackage{mathtools}
\usepackage{dsfont}
\usepackage{amsthm}

\newtheorem{theorem}{Theorem}[section]

\newtheorem{definition}[theorem]{Definition}

\newtheorem{remark}[theorem]{Remark}

\newcommand{\R}[0] {\mathbb R}

\title{Implicit biases in multitask and continual learning from a backward error analysis perspective}

\author{%
  Benoit Dherin \\
  Google\\
  dherin@google.com\\
}

\begin{document}

\maketitle

\begin{abstract}
Using backward error analysis, we compute implicit training biases in multitask and continual learning settings for neural networks trained with stochastic gradient descent. In particular, we derive modified losses that are implicitly minimized during training. They have three terms: the original loss, accounting for convergence, an implicit flatness regularization term proportional to the learning rate, and a last term, the conflict term, which can theoretically be detrimental to both convergence and implicit regularization. 
In multitask, the conflict term is a well-known quantity, measuring the gradient alignment between the tasks, while in continual learning the conflict term is a new quantity in deep learning optimization, although a basic tool in differential geometry: The Lie bracket between the task gradients. 
\end{abstract}

\section{Introduction}

Overparameterized neural networks trained to interpolate are able to generalize surprisingly well in spite of the high complexity of their hypothesis space \cite{belkin_2021}. 
One key concept to understand this phenomenon is that of implicit regularization or implicit training biases, which are quantities that are not explicitly regularized in the loss during training but by other mechanisms, guiding the network toward simpler solutions \cite{dherin2022why,Dherin2021TheGO}.
Several groups \cite{barrett2021implicit,smith2021on,rosca2023on,rosca2021discretization,fraca2021symplectic_integration,gao2023diffusion_bea,li2017stochastic,lu2020resolution_ode, rosca2023implicit_regularization,miyagawa2022equation_of_motion,cattaneo2023implicit_bias_adam,barba2021federated} have recently used a technique called {\bf Backward Error Analysis} (BEA) to compute implicit biases
as a measure of the discrepancy between an optimizer iterates and the solutions of Gradient Flow (GF), which are the unique continuous paths of steepest descent.
Because of its flexibility BEA has been used to compute optimizer implicit biases in many settings: Gradient Descent (GD) \cite{barrett2021implicit}, Stochastic Gradient Descent (SGD) \cite{smith2021on,li2017stochastic}, Momentum \cite{ghosh2023implicit}, Adam and RMSProp \cite{cattaneo2023implicit_bias_adam}, GAN's \cite{rosca2021discretization,rosca2023implicit_regularization,lu2020resolution_ode}, and diffusion processes \cite{gao2023diffusion_bea}, among others
\cite{rosca2023on,fraca2021symplectic_integration,miyagawa2022equation_of_motion,barba2021federated}.

{\bf Our contribution:} We add to this body of work by computing implicit biases in multitask learning \cite{wang2023comprehensive_survey} and continual learning \cite{soutifcormerais2023comprehensive, lange2023continual} settings optimized with SGD. In both cases, the output of BEA is a modified loss implicitly minimized by the optimizer. It consists of the original loss plus additional terms, which can be split in two parts: 1) a beneficial implicit flatness regularizer proportional to the learning rate and already observed in single-task learning (in \cite{barrett2021implicit,smith2021on} using BEA as well as with other approaches in \cite{vardi2021implicit,damian2021label,ma2021linear_stability_of_sgd}), and 2) a conflict term, due to the presence of several tasks, and which can be detrimental to both convergence and implicit flatness regularization. In multitask, the conflict term is the inner product between the task gradients, creating an implicit propensity in the learning dynamics to seek misaligned task gradients, which is known to be detrimental and needs to be mitigated \cite{yu2020gradient_surgery,wang2021gradient_vaccine,lee2022sequential}. In continual learning, the conflict term is the Lie bracket \cite{lee2012smooth_manifolds} between the task gradients whose non-vanishing may possibly be related to catastrophic forgetting \cite{soutifcormerais2023comprehensive,lange2023continual} where the performance of previous tasks degrades as new ones are learned. We hope to foster interest on Lie brackets in optimization, which is one of the basic tools in differential geometry \cite{lee2012smooth_manifolds}.

\section{Background on backward error analysis}
\label{section:bea}

To illustrate BEA, we now derive an implicit bias of SGD after a single mini-batch update by adapting the derivation for full-batch GD from \cite{barrett2021implicit}. Consider the loss $L_B(\theta)$ computed on a batch of data $B$ from a dataset $D$. At a given step, the SGD iterate is $\theta' = \theta - h\nabla L_B(\theta)$, while the solution of GF (which exactly minimizes the batch loss) is the curve $\theta(t)$ solving the differential equation $\dot \theta(t) = -\nabla L_B(\theta(t))$ with $\theta(0) = \theta$. The {\bf discretization drift} is the difference between the two, i.e., $\|\theta(h) - \theta'\|$, and it is of order $\mathcal O(h^2)$ for GD (see \cite{hairer2006geometric} for details).
BEA proposes to compute a {\bf modified equation}, in the form of GF plus corrections in terms of powers of the learning rate
\begin{eqnarray}
\dot \theta = -\nabla L_B(\theta) + h f_1(\theta) + h^2 f_2(\theta) + \cdots
\end{eqnarray}
so that the solution $\tilde \theta(t)$ of the modified equation exactly coincides with the GD iterate: $\theta' = \tilde \theta(h)$. The idea of BEA is to use the continuous modified equation to analyze the discrete optimizer.
Note that if we truncate the modified equation at order $n$ (i.e. removing the terms of order $h^n$ and higher), the discretization drift becomes of order only $\mathcal O(h^{n+1})$.
 Let us compute $f_1$ following  \cite{barrett2021implicit}: First, we expand the solution of the modified equation in a Taylor series:
\begin{equation}
    \tilde \theta(h) = \theta - h\nabla_\theta L_B(\theta) 
    + h^2\left(f_1(\theta) + \frac 14 \nabla_\theta \|\nabla_\theta L_B(\theta)\|^2\right) + \mathcal O(h^3)
\end{equation}
For $\tilde \theta(h)$ to coincide with $\theta'= \theta - h\nabla L_B(\theta)$, we need all the terms in power of $h^2$ or higher to vanish. This gives us the first correction (and recursively the higher ones too as needed; see \cite{hairer2006geometric,barrett2021implicit,rosca2023on, miyagawa2022equation_of_motion}): $f_1(\theta) = -\nabla \left(\frac 14 \|\nabla L_B(\theta)\|^2\right)$. Therefore the gradient update $\theta'$ follows a GF with drift only of order $\mathcal O(h^3)$ but for a modified loss $\widetilde L_B$, since the modified equation is of the form
\begin{equation}\label{eq:local_sgd_modified_loss}
\dot \theta = - \nabla\left( L_B(\theta) + \frac h4 \|\nabla L_B(\theta)\|^2\right) + \mathcal O(h^2),
\quad
\widetilde L_B(\theta) := L_B(\theta) + \frac h4 \|\nabla L_B(\theta)\|^2,
\end{equation}
The second term in this modified loss is a flatness bias, called {\bf Implicit Gradient Regularization} (IGR) in \cite{barrett2021implicit}, which prefers optimization paths with shallower slopes (i.e. lower gradients) guiding the trajectory toward flatter regions, very much in line with other flatness biases in SGD found by different means \cite{damian2021label,ma2021linear_stability_of_sgd,Blanc2019ImplicitRF,jastrzebski2021catastrophic}. 
We now turn to applying BEA to multitask and continual settings.

\section{Modified loss and implicit biases in Multitask learning settings}

Multitask learning trains a neural network jointly on several tasks hoping that knowledge gained from each task will transfer to the other tasks, helping generalization, and useful in case of data scarcity \cite{ruder2017overview,zhang2021survey}. However, it has been observed at times that learning multiple tasks at once can be detrimental, a circumstance attributed to the loss gradients for each task being misaligned \cite{yu2020gradient_surgery,wang2021gradient_vaccine,lee2022sequential,anguelov2020_pick_a_sign}. BEA shed some theoretical light on this since the implicit multitask dynamics of SGD given by its modified equation has a term, the conflict term, with propensity to guide the training in regions with misaligned gradients. In terms of losses, the simplest multitask setting corresponds to having two losses $L_1(\phi_1, \theta)$ for the first task
and $L_2(\phi_2, \theta)$ for the second task. The parameters $\theta$ correspond to the part of the network
that is shared between the two tasks, while $\phi_1$ and $\phi_2$ are the parameters corresponding to the two task heads. 
The training setup is to devise a global loss 
\begin{equation}\label{eq:multitask_loss}
    L_{\alpha, \beta}(\phi_1, \phi_2, \theta) := \alpha L_1(\phi_1, \theta) + \beta L_2(\phi_2, \theta)
\end{equation}
consisting on a weighted average of the two losses and update the network parameters with 
\begin{eqnarray}
\omega' = \omega - h \nabla_\omega L_{\alpha, \beta}(\phi_1, \phi_2, \theta),
\label{global_update}
\quad\textrm{with}\quad \omega = (\phi_1, \phi_2, \theta).
\end{eqnarray}
Note that the update above can be considered either as a full-batch GD update as in \cite{barrett2021implicit}, or a single-step batch update in SGD within an epoch as in Section \ref{section:bea}.

\begin{theorem}
At any given SGD step the two-task iterate \eqref{global_update} follows an exact GF $\dot \omega = -\nabla_\omega \widetilde L_{\alpha, \beta}(\omega)$ with a modified loss
\begin{equation}\label{eq:two_task_mod_loss}
\widetilde L_{\alpha, \beta}  =  L_{\alpha, \beta} 
+ \frac{h\alpha^2}4 \|\nabla_{\omega_1} L_1 \|^2 
+ \frac{h\beta^2}4  \|\nabla_{\omega_2} L_2 \|^2 
+ \frac{h\alpha\beta}2  \langle \nabla_\theta L_1, \nabla_\theta L_2 \rangle, 
\end{equation}
with discretization drift $\|\tilde \omega(h) - \omega'\|$ of order $\mathcal O(h^3)$, where $\tilde \omega$ is the solution of the modified GF starting at $\omega$ and where $\omega_i:=(\phi_i, \theta)$ for $i=1,2$,.
\end{theorem}

\begin{proof}
This follows from an immediate application of the supervised modified loss and modified equation in \eqref{eq:local_sgd_modified_loss} to the special form of the multitask loss \eqref{eq:multitask_loss}.
\end{proof}

\paragraph{Interpretation:} The modified loss \eqref{eq:two_task_mod_loss} has two implicit biases: a IGR term and a conflict term 
\begin{equation*}
\mathrm{IGR} 
= \frac{\alpha^2h}4 \|\nabla_{\omega_1} L_1\|^2
+ \frac{\beta^2h} 4 \|\nabla_{\omega_2} L_2\|^2,
\quad\quad
\mathrm{conflict} = 
\frac{h\alpha\beta}2  \langle \nabla_\theta L_1, \nabla_\theta L_2 \rangle.
\end{equation*}

\emph{The IGR term is beneficial:} It consists of the sum of two implicit flatness regularizers for each task loss proportional to the learning rate $h$ as in the single-task case, where it has been shown  to be beneficial, guiding optimization paths toward flatter regions with greater generalization power \cite{barrett2021implicit,smith2021on}.  
\emph{The conflict term  can be detrimental:} The implicit dynamics from the multitask modified equation encourages this term to become negative possibly at the expense of the original losses or the IGR terms. Regions where the conflict term can be negative are also regions where the gradients of the two losses w.r.t. the shared parameters are in opposite directions, creating smaller updates for the shared parameters, resulting in possibly stalled learning. It turns out that mechanisms preventing this conflict term to become negative (e.g. by projection \cite{wang2021gradient_vaccine,anguelov2020_pick_a_sign,yu2020gradient_surgery} or direct regularization \cite{lee2022sequential}) have been identified and used successfully to improve train and test performance in multitask settings.

\section{Modified loss and implicit biases in continual learning settings}

Continual learning is concerned with learning from a data distribution that is changing over time with tasks corresponding to locally stationary phases of the evolution \cite{soutifcormerais2023comprehensive,lange2023continual}. One of its major issues is catastrophic forgetting, when the updates from latter tasks degrade the performance on earlier tasks. While catastrophic forgetting is an issue for all modern approaches \cite{soutifcormerais2023comprehensive,lange2023continual, prabhu2020GDumb}, its causes remain unclear. As we will see, the BEA modified equation for continual learning may help shed some new light on this issue.
Namely, consider the continual learning setting where we perform two successive SGD updates $\theta_1 = \theta_0  - h \nabla L_1(\theta_0)$ and $\theta_2 = \theta_1  - h \nabla L_2(\theta_1)$, with the two losses computed on two successive batches from a changing data distribution.
Using BEA, we want to first compute a modified loss whose continuous minimization approximates well the two successive updates. Then we want to identify possibly detrimental terms in the modified equation that may be responsible for a decreased performance on the first batch by the second update. It turns out that such a detrimental term pops up, controlled by the Lie bracket of the two batch gradients:
\begin{definition}
Given two vector fields on $\R^n$, that is, two differentiable functions $F, G:\R^n\rightarrow \R^n$, their Lie Bracket is the vector field $[F, G]:\R^n\rightarrow \R^n$ defined as follows
\begin{equation}
    [F, G] (\theta) = \nabla G(\theta) F(\theta) - \nabla F(\theta) G(\theta),
\end{equation}
where $\nabla G(\theta)$ and $\nabla F(\theta)$ are the Jacobians of the vector fields.
\end{definition}
Lie brackets are fundamental tools in differential geometry \cite{lee2012smooth_manifolds}.
They help quantify how flows intertwine. For instance, if the Lie bracket between loss gradients for different tasks vanish, i.e.,  $[\nabla L_1, \nabla L_2]=0$, this implies that their gradient flows \emph{commute}: Following the gradient flow of first $L_1$ and then $L_2$ yields the same result as the reverse \cite{lee2012smooth_manifolds}, with their flows somehow spanning "non-interacting" subspaces. 
The next theorem states that when this happens two consecutive SGD updates as above can be approximated by GF for a modified loss of the form:
\begin{equation}\label{modified_loss}
    \widetilde L_{1,2}(\theta) = L_1(\theta) + L_2(\theta) + \frac h4 \|\nabla L_1(\theta)\|^2 + \frac h4 \|\nabla L_2(\theta)\|^2,
\end{equation}
where the IGR terms encourage the learning trajectory toward flatter regions for each task. Note that flatness preservation between tasks seems helpful to combat catastrophic forgetting \cite{sanket2023pretraining_in_lifelong_learning,shi2023create}. However, when $[\nabla L_1, \nabla L_2]\neq 0$, a term of order $h$   in the modified equation (theorem below) and proportional to the Lie bracket can potentially disrupt that implicit flatness regularization induced by the modified loss above.  Since it is the only term of order $h$  that can do so, we conjecture that the non-vanishing of the Lie bracket between loss gradients pertaining to  different tasks may be linked to catastrophic forgetting in continual learning. The following theorem gives an exact description of how this Lie bracket affects the implicit gradient regularization dynamics:

\begin{theorem}
Consider two consecutive mini-batch gradient descent updates $\theta_1$ and $\theta_2$ as above. The solution $\tilde \theta(t)$ of the modified equation 
\begin{equation}\label{modified_equation}
   \dot \theta (t) = - \nabla  \widetilde L_{1,2}(\theta(t)) + \frac h2 [\nabla L_1, \nabla L_2](\theta(t)) + O(h^2),
\end{equation}
where $ \widetilde L_{1,2}$ is the modified loss in Equation \eqref{modified_loss} follows the composition iterate $\theta_2$ with discretization drift $\|\tilde \theta(h) - \theta_2\|$ of order $\mathcal O(h^3)$.
\end{theorem}

\begin{proof}
To simplify the notation, let us start with two consecutive Euler updates for general vector fields $F$ and $G$. First we consider an Euler update for the first vector field: $\theta_1 = \theta_0 + h F(\theta_0)$. Then, we compose this update with an Euler step in the direction of the second vector field $G$ and expand the result into a Taylor's Series:
\begin{eqnarray*}
\theta_2 
& = & \theta_1 + h G(\theta_1) \\
& = & \theta_0 + h F(\theta_0)  + h G(\theta_0 + h F(\theta_0))\\
& = & \theta_0 + h (F(\theta_0) + G(\theta_0)) + h^2 \nabla G(\theta_0)F(\theta_0) + \mathcal O(h^3).
\end{eqnarray*}
Now, we want to find a modified equation of the form
\begin{equation}
    \dot \theta = H_0(\theta) + h H_1(\theta) + h^2 H_2(\theta) + \cdots 
\end{equation}
whose solution starting at $\theta_0$ coincides with $\theta_2$ after time $h$. For that, we can compute the Taylor expansion of the modified equation solution and compare the powers in $h$ to obtain recursive formulas for the $H_i$'s. It is easy to verify the the first orders of the solution Taylor's Series are given by the following expression:
\begin{equation}
    \theta(h) = \theta_0 + h H_0(\theta_0) + h^2(H_1(\theta_0) + \frac12 \nabla H_0(\theta_0) H_0(\theta_0)) + \mathcal O (h^3).
\end{equation}
To have that $\theta_2 = \theta(h)$ at first order, we obtain the following condition: $H_0(\theta_0) = F(\theta_0) + G(\theta_0)$. This yields for the second order the following condition:
\begin{equation}
    H_1(\theta_0) + \frac12 \nabla H_0(\theta_0) H_0(\theta_0) = \nabla G(\theta_0) H(\theta_0).
\end{equation}
Using the first order condition and expanding, we immediately obtain
\begin{equation}
    H_1(\theta_0) = -\frac12\left(\nabla F(\theta_0) F(\theta_0) + \nabla G(\theta_0) G(\theta_0)\right) + \frac12[F, G](\theta_0),
\end{equation}
where the last term is the Lie bracket between $F$ and $G$ evaluated at $\theta_0$. Now if we specialized for the gradient fields $F = - \nabla L_1$ and $G = - \nabla L_2$, we obtain that
\begin{eqnarray}
H_0(\theta) & = & -\nabla(L_1 + L_2) \\
H_1(\theta) & = & -\nabla \left(\frac 14 \|\nabla L_1\|^2 + \frac 14 \|\nabla L_2\|^2\right) + \frac 12 [\nabla L_1, \nabla L_2],
\end{eqnarray}
which concludes the theorem. 
\end{proof}

\begin{remark} Observe that when the two losses $L_1$ and $L_2$  come from batches pertaining to the same task (i.e., close to i.i.d.)  their gradients are more likely to be aligned. By the anti-symmetry of the Lie bracket, $[\nabla L_1, \nabla L_2]$  is then more likely to be close to zero. However, when the data distribution changes, creating a sharp contrast between the two task loss-gradients then the Lie bracket $[\nabla L_1, \nabla L_2]$  is likely to be large. This seems to be relevant to the \emph{stability gap} noticed in \cite{lange2023continual}, when a large and sudden decrease in performance is  observed after the first update for the second task. 
\end{remark}

\section{Conclusion}
We computed implicit biases in multitask and continual learning optimized with SGD using backward error analysis. These biases are local, measuring the discrepancy between one step of SGD and gradient flow on the batch loss. In both cases we found a beneficial flatness bias proportional to the learning rate and preferring smaller slopes on the loss surface for each task along the learning trajectories similar to single-task supervised learning \cite{barrett2021implicit,smith2021on}. We also found a detrimental implicit bias in both cases (due to the presence of several tasks and which we called conflict term) that has the potential to steer the learning dynamics away from the flatter regions with higher generalization power. For multitask learning, the detrimental implicit bias is controlled by the inner product of the task loss-gradients $\langle \nabla L_1, \nabla L_2 \rangle$, which is a known key quantity in multitask learning already (e.g., \cite{yu2020gradient_surgery,wang2021gradient_vaccine,lee2022sequential,anguelov2020_pick_a_sign}). For continual learning the detrimental bias is a new quantity, the Lie bracket $[\nabla L_1, \nabla L_2]$ between the task loss-gradients measuring how much their respective gradient flows span independent regions of the parameter space. Despite their wide use in many areas of mathematics, Lie brackets are new to deep learning optimization to the best of our knowledge. We hope this work will help foster the use of backward error analysis in deep learning, and serve as a theoretical motivation to devise methods relying on Lie brackets in continual learning.

\newpage

\begin{ack}
We would like to thank Mihaela Rosca, Maxim Neumann, Michael Munn, Javier Gonzalvo, David Barrett, and Hossein Mobahi for helpful discussions and feedback. We would also like to thank Patrick Cole for his support.
\end{ack}

\medskip

\bibliographystyle{unsrt}

\newpage
\clearpage


\end{document}